\documentclass[a4paper,fleqn]{cas-dc}

\usepackage{cite}
\usepackage{amsmath,amsfonts}
\usepackage{array}

\usepackage{amsthm}
\usepackage{amsmath}
\usepackage{graphicx}
\usepackage{subfigure}
\usepackage[square,numbers,sort&compress]{natbib}
\newtheorem{theorem}{Theorem}
\usepackage[ruled,linesnumbered]{algorithm2e}
\hypersetup{
    colorlinks=true,
    citecolor=blue,
    linkcolor=blue,
    urlcolor=blue,
}

\def\tsc#1{\csdef{#1}{\textsc{\lowercase{#1}}\xspace}}
\tsc{WGM}
\tsc{QE}
\tsc{EP}
\tsc{PMS}
\tsc{BEC}
\tsc{DE}

\begin{document}
\let\WriteBookmarks\relax
\def\floatpagepagefraction{1}
\def\textpagefraction{.001}

\shorttitle{Sight View Constraint for Robust Point Cloud Registration}

\shortauthors{Yaojie Zhang et~al.}

\title [mode = title]{SVC: Sight View Constraint for Robust Point Cloud Registration}                      

\tnotetext[1]{This work was supported by the National Key R\&D Program of China (2023YFB4705002), the National Natural Science Foundation of China (U20A20283), the Guangdong Provincial Key Laboratory of Construction Robotics and Intelligent Construction (2022KSYS013), and the Science and Technology Cooperation Special Project of Hubei Province and the Chinese Academy of Sciences (2023-01-08).}


%
\author[1,2]{Yaojie Zhang}[type=editor,
                        auid=,bioid=,
                        prefix=,
                        role=,
                        orcid=0009-0006-5752-0420]



\ead{yj.zhang1@siat.ac.cn}


\credit{Conceptualization of this study, Methodology, Software, Writing-original draft}


\affiliation[1]{organization={Shenzhen Institute of Advanced Technology, Chinese Academy of Sciences},
    city={Shenzhen},
    postcode={518055}, 
    state={},
    country={China}}
    
\affiliation[2]{organization={University of Chinese Academy of Sciences},
    addressline={}, 
    city={Beijing},
    postcode={100049}, 
    country={China}}

\author[1,2]{Weijun Wang}[style=chinese]
\credit{Investigation, Writing, Supervision}

\author[1,2]{Tianlun Huang}[%
   role=,
   suffix=,
   ]
\credit{Formal analysis, Revising}
\author[1,2]{Zhiyong Wang}[style=chinese]
\credit{Visualization, Investigation}
\credit{Data curation, Writing - Original draft preparation}


\author%
[1,2,3]
{Wei Feng}[type=,
                        auid=,bioid=,
                        prefix=,
                        role=IEEE Senior Member,
                        orcid=0000-0002-9845-999X]
\credit{Investigation, Supervision}
\cormark[1]
\ead{wei.feng@siat.ac.cn}
\ead[URL]{}

\affiliation[3]{{Guangdong Provincial Key Laboratory of Construction Robotics and Intelligent Construction},
    addressline={}, 
    city={Shenzhen},
    postcode={518055}, 
    state={},
    country={China}}

\cortext[cor1]{Corresponding author}



\begin{abstract}
Partial to Partial Point Cloud Registration (partial PCR) remains a challenging task, particularly when dealing with a low overlap rate. In comparison to the full-to-full registration task, we find that the objective of partial PCR is still not well-defined, indicating no metric can reliably identify the true transformation. We identify this as the most fundamental challenge in partial PCR tasks. In this paper, instead of directly seeking the optimal transformation, we propose a novel and general Sight View Constraint (SVC) to conclusively identify incorrect transformations, thereby enhancing the robustness of existing PCR methods. Extensive experiments validate the effectiveness of SVC on both indoor and outdoor scenes. On the challenging 3DLoMatch dataset, our approach increases the registration recall from 78\% to 82\%, achieving the state-of-the-art result. This research also highlights the significance of the decision version problem of partial PCR, which has the potential to provide novel insights into the partial PCR problem.


\end{abstract}




\maketitle

\section{Introduction}
\label{sec:intro}
Point cloud registration (PCR) emerges as a critical and foundational challenge in 3D computer vision. The objective of the PCR task is to determine an optimal six-degree-of-freedom (6-DoF) pose transformation, ensuring precise alignment of input point clouds. 
Using point-to-point feature correspondences is a popular and robust solution to the PCR problem.

Point cloud registration (PCR) can be categorized based on the overlap ratio into (1) full to full, (2) partial to full, and (3) partial to partial registration. Currently, partial to partial registration (partial PCR) remains a challenging issue in the PCR field \cite{huang2021predator}, especially when the overlap rate is low. In the case of correspondence-based PCR methods, this challenge is considered primarily from correspondences with a large number of outliers when the overlap rate is low. Many previous studies have made commendable contributions to address this challenge by proposing distinct descriptors \cite{huang2021predator, zeng20173dmatch, choy2019fully,  qin2022geometric} and stable outlier rejection methods \cite{zhou2016fast, bai2021pointdsc, chen2023sc, zhang20233d}. However, in this paper, we argue that the high outlier rate issue is a significant but not the most fundamental challenge of the partial PCR task. Compared with the full-to-full and partial-to-full registration tasks, we find that partial PCR is still not a well-defined problem when the overlap rate is low, leading to unreliable model selection. We consider this as the primary reason contributing to the challenge of the partial PCR task.

\begin{figure*}[!t]
  \centering
  \includegraphics[width=0.98\linewidth]{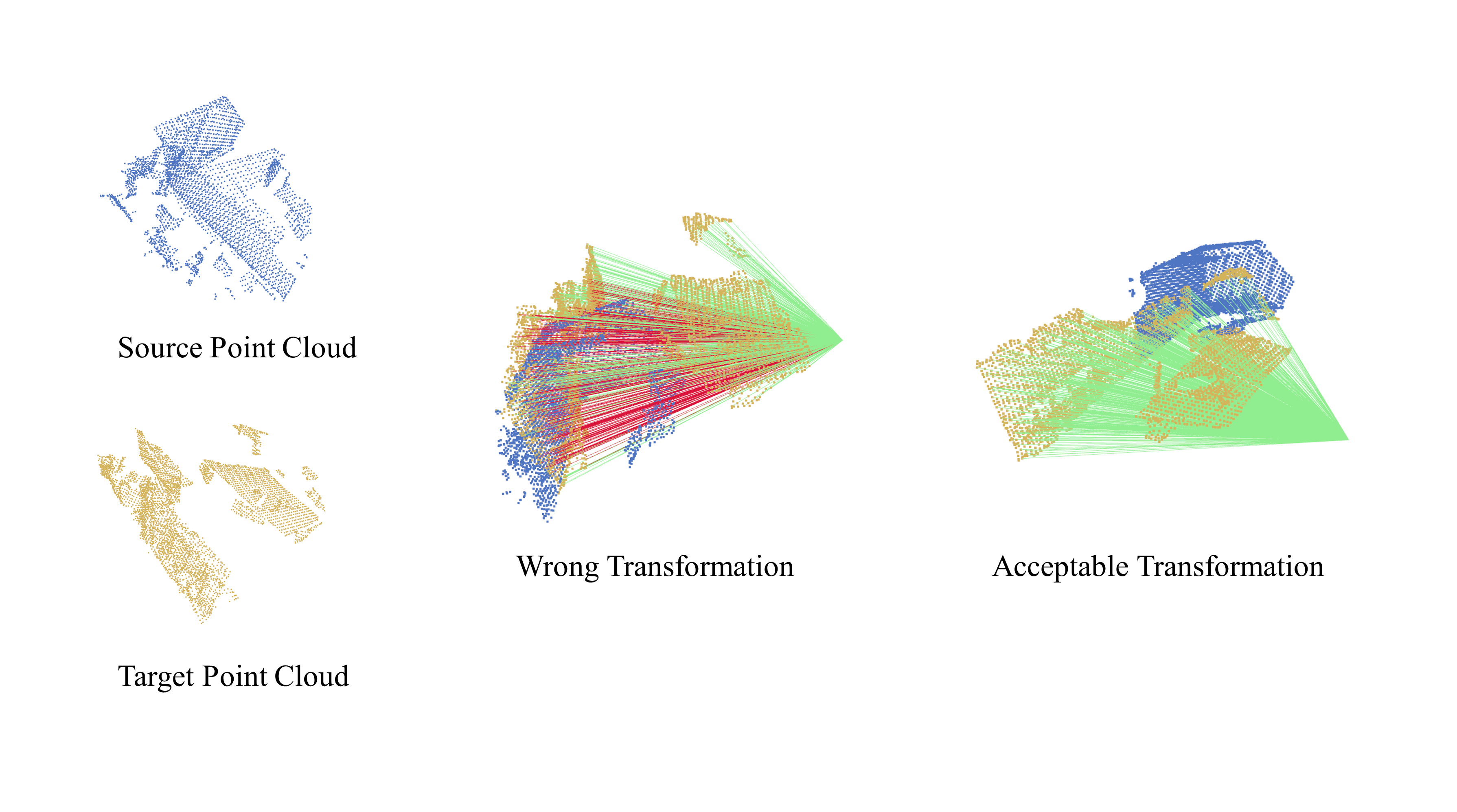}
  \caption{\textbf{An illustration of the core idea of SVC.} The line (green or red) represents the sight view line from the sensor viewpoint to the target point. The red line means there exist source points in the line so that the sight view is blocked. The green line means the sight view is not blocked. If there exist lots of red lines, then the estimated transformation is wrong.}
  \label{fig:core_idea}
\end{figure*}

The "optimal" alignment of partially overlapping point clouds, which is the objective of the partial PCR task, is still not well-defined. For example, given two partially overlapping point clouds and multiple hypotheses of pose transformations that include only one correct transformation, there is no known metric that reliably identifies the correct one as the "optimal" transformation among all hypotheses. 
In the method Sec.\ref{sec:method}, we will elaborate on the challenges of finding such a metric. When the overlap rate is extremely low, even whether such a metric exists remains in doubt. 

Recently, the \cite{xing2024efficient} used viewpoint deviation distance to identify the optimal metric, achieving leading performance. However, the optimal result of this metric still can not guarantee the optimal transformation. In this paper, instead of directly finding the optimal transformation, we utilize the sight view constraint to identify the definitively incorrect transformations. An illustration of our method is given in Fig.\ref{fig:core_idea}. The core idea of our method is simple and rigorous: For a pair of 3D scan point clouds, we obtain the transformed source point cloud using the estimated transformation. Then we can assert the following theorem:

\begin{theorem}
In a static environment, the transformed source point cloud cannot block the line of sight between the target point cloud and the sensor. Otherwise, the estimated transformation is incorrect.
\label{theorem:only_one}
\end{theorem}
\begin{proof}
It is straightforward to prove this theorem using proof by contradiction. Assuming the estimated transformation is correct, there exist new points in front of the target points, and then these target points are blocked by new points, resulting in a difference between new target point clouds and the existing ones. This contradicts the assumption of a static environment, proving that the transformation is incorrect. 
\end{proof}
By applying this constraint, we can narrow down the range of pose transformation hypotheses, significantly enhancing the registration performance of existing PCR methods.
In summary, our contributions are:
\begin{enumerate}
    \item We propose a strict and general sight view constraint (SVC) that could identify incorrect transformations. The SVC can significantly improve the robustness of existing PCR methods.
    \item By analyzing the objective of the partial PCR task, we underscore the decision version of partial PCR problem as the fundamental challenge. Based on this, we state the regime of PCR methods as (1) generating effective hypotheses including the correct one and (2) identifying the correct transformation readily.
\end{enumerate}

\section{Related Work}
\label{sec:related_work}
The Correspondence-based PCR methods mainly include two steps, initial correspondence generating and model fitting. In this paper, our research is mainly about the model fitting part.

\subsection{Model Fitting}
\subsubsection{RANSAC-based methods}
The model fitting aims to find the best pose transformation for the initial correspondence and plays a key role in PCR methods. The model fitting normally follows the Hypotheses Generation and Selection pipeline starting with the pioneer RANSAC \cite{fischler1981random}. In past decades, many of its variants \cite{chum2005matching,barath2018graph,quan2020compatibility,barath2019magsac} have been proposed to improve time efficiency and robustness performance. One common challenge of the RANSAC and its variants is low inlier ratios. To improve the RANSAC performance, GORE \cite{bustos2017guaranteed} and QGORE \cite{li2023qgore} can be utilized to increase the inlier ratio by rejecting most true outliers. 
\subsubsection{Spatial compatibility methods}
Due to the time complexity of the RANSAC method, various RANSAC-free methods have been explored. The Spatial compatibility is widely applied in point cloud registration. It utilizes correspondence-wise spatial constraints and transfers the outlier rejection problem to a maximum clique problem in graph theory. Clipper \cite{lusk2021clipper} and Teaser \cite{yang2020teaser} introduce a graph-theoretic framework for outliers rejection. SC2 \cite{chen2022sc2} presents a second-order spatial compatibility measure, enhancing the distinctiveness of clustering compared to the original measure. MAC \cite{zhang20233d} relaxes the maximum clique constraint to a maximal clique constraint, enabling the extraction of more local information from a graph. Several deep-learning methods also leverage the spatial compatibility (SC) technique. PointDsc \cite{bai2021pointdsc} incorporates a non-local module guided by the SC for improved performance. DHVR \cite{lee2021deep} generates hypotheses for deep Hough voting using the SC-validated tuples.

\subsection{Metrics for Model fitting}
Model fitting methods typically aim to select the "optimal" model based on the best score according to a specific evaluation metric. Common metrics include inlier count (IC), mean average error (MAE), and mean square error (MSE) \cite{yang2021toward}. FGR \cite{zhou2016fast} uses the scaled Geman-McClure function as the robust penalty to define the metric, taking into account both inlier count and mean average error. SC2++ \cite{chen2023sc} propose a new feature and spatial consistency constrained Truncated Chamfer Distance (FS-TCD) metric, which incorporates feature descriptor information for improvement. The VDIR \cite{xing2024efficient} identifies the optimal transformation with the minimum viewpoint deviation distance. However, as discussed in Sec.\ref{sec:intro}, the sight view is only suitable for identifying incorrect transformation rather than serving as a metric.

For correspondence-based PCR methods, these metrics typically operate on initial correspondences generated by matching feature descriptors. This approach has advantages in terms of achieving global optimality and time efficiency \cite{zhou2016fast}. However, it also makes model fitting methods dependent on the quality of initial correspondences. In this case, multiple deep-learning descriptors, e.g. FCGF \cite{choy2019fully}, predator \cite{huang2021predator}, Geotransformer \cite{10076895}, are designed to improve the registration performance. In other words, the current model fitting procedure can not guarantee the selected "optimal" model is correct for the PCR task. Especially when the overlap is low, the performance of the model fitting methods significantly varies using different descriptors.

Currently, no metric for model-fitting can reliably determine the correct result for low overlapping PCR tasks. Due to the limitations of model fitting performance, greater emphasis is placed on the generation of initial correspondence sets, with model fitting primarily serving as a means of outlier rejection.  Nevertheless, under full-to-full and partial-to-full registration scenarios, the iterative closest point (ICP) based method Go-ICP \cite{yang2015go} can find the global optimal transformation readily using a branch-and-bound scheme. In the following section, we will further analyze the differences between them and propose our method.

\section{Method}
\label{sec:method}
\begin{figure*}[!t]
  \centering
  \includegraphics[width=0.98\linewidth]{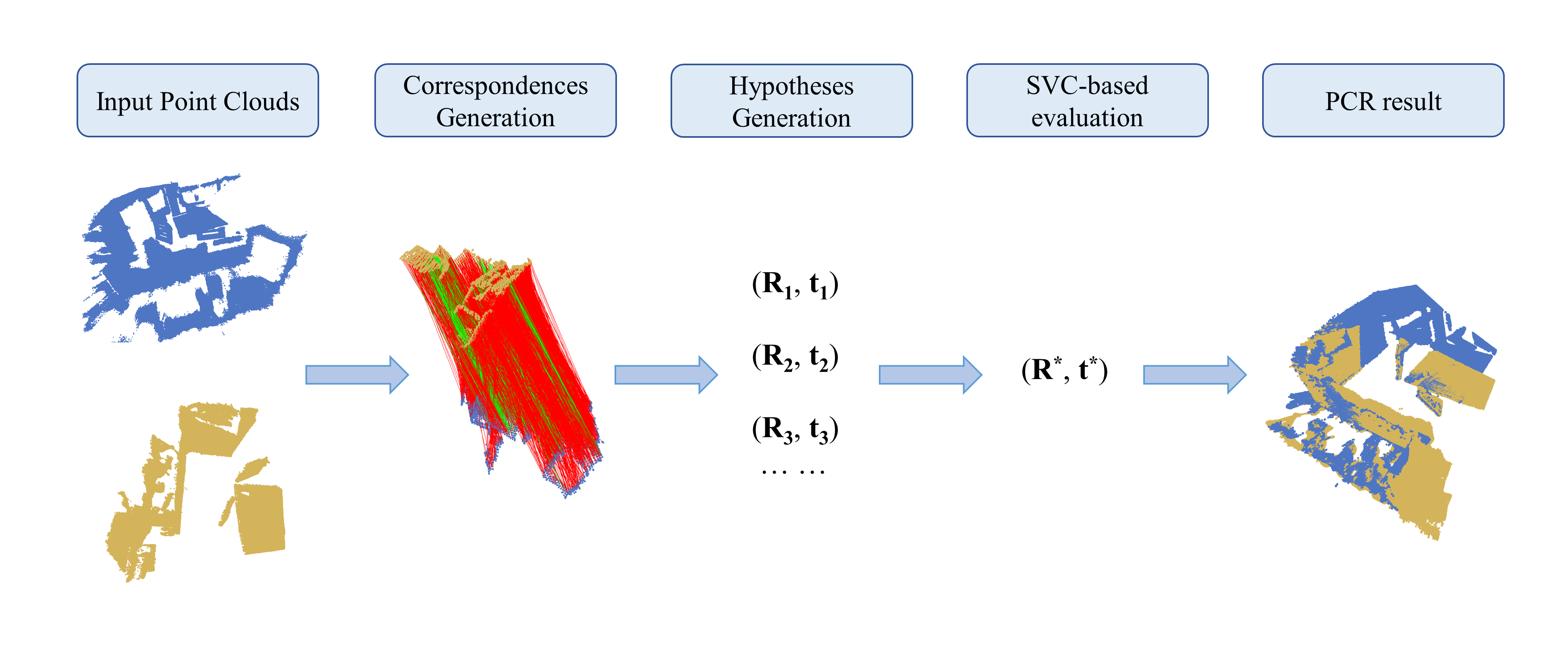}
  \caption{\textbf{A simple pipeline of the correspondence-based PCR method.} (1) Generate initial correspondences according to feature descriptors. (2) Identify inliers and outliers to generate multiple transformation hypotheses. (3) Using SVC-based evaluation to select the optimal transformation. (4) Merge point clouds according to the transformation.}
  \label{fig:overall_pipeline}
\end{figure*}
\subsection{Reminder of the PCR task}
Given two point clouds to be aligned: source point cloud $\mathcal{P} = \{ \boldsymbol{p_i} \in \mathbb{R}^3\ |\ i = 1, ..., M \}$ and target point cloud $\mathcal{Q} = \{ \boldsymbol{q_j} \in \mathbb{R}^3\ |\ j = 1, ..., N \}$. The objective is to recover an optimal 3D rigid transformation with rotation $\mathbf{{R}^{*}} \in \mathcal{SO}(3)$ and translation $\mathbf{{t}^*} \in \mathbb{R}^3$, which minimize the cost with a specific point clouds dissimilarity evaluation metric:

\begin{equation}
\small
d(\mathcal{\hat{P}}, \mathcal{Q}): (\mathcal{\hat{P}}, \mathcal{Q}) \mapsto \mathbb{R},\ where\ \mathcal{\hat{P}} = \{ \boldsymbol{\hat{p_i}} = \mathbf{R} \boldsymbol{p_i}+ \mathbf{t}\ |\ \boldsymbol{p}_i \in \mathcal{P}\}.
  \label{eq:abstract_metric}
\end{equation}

For the full-to-full and the partial-to-full overlap situation, the following Nearest Neighbor (NN) metric can be applied to distinguish the optimal transformation readily \cite{yang2015go}:
\begin{equation}
\small
d_{NN}(\mathcal{\hat{P}}, \mathcal{Q}) =\sum_{i=1}^M ||\boldsymbol{\hat{p_i}} - \boldsymbol{q_i^*}||_{2}^{2} = \sum_{i=1}^M ||\mathbf{R} \boldsymbol{p_i} + \mathbf{t} - \boldsymbol{q_i^*}||_{2}^{2},
  \label{eq:NN_L2_error}
\end{equation}
where $\boldsymbol{q_i^*}$ represents the nearest point of the transformed point $\boldsymbol{\hat{p_i}}$ in $\mathcal{Q}$:

\begin{equation}
\small
\boldsymbol{q_i^*} = NN(\boldsymbol{\hat{p_i}}, \mathcal{Q}) = \mathop{\arg\min}_{\boldsymbol{q} \in \mathcal{Q}}\ ||\boldsymbol{\hat{p_i}} - \boldsymbol{q}||_2.
  \label{eq:nearest_index}
\end{equation}
Under this condition, there is no outlier since all source points have their correct correspondences. Thus the global optimal of Eq.\ref{eq:NN_L2_error} guarantees a correct transformation.

For the partial-to-partial overlap situation, the NN metric is unusable since the overlap region is unclear. Under this condition, the primary objective of the metric is to distinguish the overlap part and non-overlap part also called inliers and outliers. The evaluated overlap rate should be higher than one threshold (e.g. $\eta_1$ = 10\%) so that the PCR task makes sense. Under this constraint, we can have the range of the estimated transformation $(\mathbf{\hat{R}},\mathbf{\hat{t}}) \in Range(\mathcal{P}, \mathcal{Q}, \eta_1)$:

\begin{equation}
\small
Range(\mathcal{P}, \mathcal{Q}, \eta_1) := \{\sum_{i=1}^m [||\mathbf{R} \boldsymbol{p_i} + \mathbf{t} - \boldsymbol{q_i^*}||_{2} < \tau] > \eta_1 M \},
  \label{eq:range_of_RT}
\end{equation}
where the $\tau$ represents the tolerance of inliers, the [$\cdot$] is denoted as the Iverson bracket and returns 1 if the statement is true. Typical metrics for the part overlap PCR including Maximum Inlier Count (IC), MSE, and MAE. Take the Maximum Inlier Count as an example:

\begin{align}
\small
    (\mathbf{\hat{R}},\mathbf{\hat{t}}) &=\mathop{\arg\max}_{(\mathbf{R},\mathbf{t}) \in Range(\mathcal{P}, \mathcal{Q}, \eta_1)} d_{IC}(\mathcal{\hat{P}}, \mathcal{Q})
    \label{eq:IC_metric}\\
    &=\mathop{\arg\max}_{(\mathbf{R},\mathbf{t}) \in Range(\mathcal{P}, \mathcal{Q}, \eta_1)} \sum_{i=1}^m [||\mathbf{R} \boldsymbol{p_i} + \mathbf{t} - \boldsymbol{q_i^*}||_{2} < \tau].
    \label{eq:detailed_IC_metric}
\end{align}
Please note the optimal of this metric can not promise a correct transformation especially when the overlap rate is low. Take the inlier count metric as an example, there may exist outliers that can form a larger overlapping area than inliers. When the overlap rate is over 50\%, the condition is better but still can not guarantee a definite correct result.

\subsection{Sight view Constraint}
\label{sec:svc}
In this subsection, we utilize the sight view constraint (SVC) to figure out impossible transformations to narrow the $Range(\mathcal{P}, \mathcal{Q}, \eta_1)$ to gain a more robust transformation result. The constraint is mainly concerned with whether the non-overlap part of the transformed $\mathcal{\hat{P}}$ will block the sight of the target point cloud. So we first get the non-overlap set $\mathcal{\hat{P}}_{non}$ from $\mathcal{\hat{P}}$ using:
\begin{equation}
\small
\mathcal{\hat{P}}_{non} = 
 \{||\hat{\boldsymbol{p_i}} - \boldsymbol{q_i^*}||_{2} > \tau\ |\ \hat{\boldsymbol{p_i}} \in \mathcal{\hat{P}}, \boldsymbol{q_i^*} = NN(\boldsymbol{\hat{p_i}}, \mathcal{Q})\}
   \label{eq:non_overlap_P}
\end{equation}

The transformed $\mathcal{\hat{P}}$ shared the same Coordinate System with target point cloud $\mathcal{Q}$.
First, we project the 3D coordinates of $\mathcal{\hat{P}}$ and $\mathcal{Q}$ to a unit sphere centered on the Origin of the target point cloud $\mathcal{Q}$:

\begin{equation}
\small
\mathcal{\hat{P}}_{sphere} = \{Proj(\boldsymbol{\hat{p}_i})\ |\ \boldsymbol{\hat{p}_i} \in \mathcal{\hat{P}}_{non} \},\ \mathcal{Q}_{sphere} = \{Proj(\boldsymbol{q_j})\ |\ \boldsymbol{q_j} \in \mathcal{Q} \}
  \label{eq:projection}
\end{equation}
where the function $Proj(\cdot)$ is defined as:
\begin{equation}
\small
Proj(\mathbf{x}): \mathbb{R}^3 \rightarrow \mathbf{S^2}: \mathbf{x} \mapsto \frac{\mathbf{x}}{||\mathbf{x}||_2}.
  \label{eq:projection_function}
\end{equation}
Please note the Origin of the target point cloud represents the sensor position. If the sensor position is not (0,0,0) then we need to perform a translation first. Then we find the nearest neighbor of the target points $Proj(\boldsymbol{q_j}) \in \mathcal{Q}_{sphere}$ in $\mathcal{\hat{P}}_{sphere}$, if the dot product of two points greater than a threshold, then these two points are considered in the same sight. We take the target points that have the same sight correspondence as the region of interest (roi) in the following.

\begin{equation}
\small
\mathcal{Q}_{roi}:=\{\langle Proj(\boldsymbol{q_j}),Proj(\boldsymbol{\hat{p}_j})^*\rangle > T_{threshold})\ |\ \boldsymbol{q_j} \in \mathcal{Q}\}
  \label{eq:Q_roi}
\end{equation}
where
\begin{equation}
\small
Proj(\boldsymbol{\hat{p}_j})^* = NN(Proj(\boldsymbol{q_j}), \mathcal{\hat{P}}_{sphere}).
  \label{eq:NN_of_projection}
\end{equation}

With a bit of notion abuse, we use the $(\boldsymbol{\hat{p}_j})^*$ represents the original points in $\mathcal{\hat{P}}_{non}$ of the $Proj(\boldsymbol{\hat{p}_j})^*$ to make a distinction with  Eq.\ref{eq:nearest_index}.
For the point $\boldsymbol{q_{j}}$ in the target point cloud $\mathcal{Q}_{roi}$, it is considered as blocked if the transformed source point  $(\boldsymbol{\hat{p}_j})^*$ on the same line of sight but closer to the sensor. Then we can compute the following blocked points count (BC) metric:

\begin{equation}
\small
d_{BC}(\mathcal{\hat{P}}, \mathcal{Q}_{roi}) = \sum_{\boldsymbol{q_{j}} \in \mathcal{Q}_{roi}}[||\boldsymbol{q_{j}}||_2 - ||(\boldsymbol{\hat{p}_j})^*||_2 > \tau],
  \label{eq:BC_metric}
\end{equation}
Where the $\tau$ is the same as in Eq.\ref{eq:range_of_RT}.
Please remind the transformed $\mathcal{\hat{P}}$ is depend by $\mathbf{R},\mathbf{t}$, so the $d_{BC}(\mathcal{\hat{P}}, \mathcal{Q}_{roi})$ is also depend by $\mathcal{P}, \mathcal{Q}, \mathbf{R}, \mathbf{t}$. We can use this metric to identify the false transformation if the blocked points count exceeds a certain number. Then we can narrow the $Range(\mathcal{P}, \mathcal{Q}, \eta_1)$ in Eq.\ref{eq:range_of_RT} into following $Range_{svc}(\mathcal{P}, \mathcal{Q}, \eta_1, \eta_2)$:

\begin{equation}
\text{\small $Range_{svc} = \{d_{BC}(\mathcal{\hat{P}}, \mathcal{Q}_{roi})<\eta_2 N\ |\ (\mathbf{{R}},\mathbf{{t}}) \in Range(\mathcal{P}, \mathcal{Q}, \eta_1)\}
$
}
    \label{eq:new_range}
\end{equation}

Currently, the SVC only narrows the range of transformations by rejecting incorrect ones. To select the optimal transformation, we need to combine the SVC with an existing metric. Applying this new range with the IC metric, we can re-write the Eq.\ref{eq:IC_metric} as follows:

\begin{equation}
\small
(\mathbf{\hat{R}},\mathbf{\hat{t}}) = \mathop{\arg\max}_{(\mathbf{R},\mathbf{t}) \in Range_{svc}} d_{IC}(\mathcal{\hat{P}}, \mathcal{Q})
  \label{eq:IC}
\end{equation}
One of the intriguing features of SVC is that it utilizes both non-overlapping and overlapping areas to make a decision. While the other metrics only focus on the overlapping parts or inliers. Please note the SVC is general for PCR tasks, we can simply replace the $d_{IC}(\mathcal{\hat{P}}, \mathcal{Q})$ with other metrics and plug-and-play in PCR methods. We will introduce the implementation of SVC in the subsequent section.


\subsection{The SVC-based algorithm}

The general pipeline for the correspondence-based PCR method is shown in Fig.\ref{fig:overall_pipeline}, and the SVC takes effect at the last hypotheses evaluation procedure. The detailed implementation of SVC-based evaluation and SVC algorithm are shown in Alg.\ref{alg:evaluation} and Alg.\ref{alg:SVC}. For the SVC-based evaluation, we apply a double-check sight view constraint for both the estimated transformation ($\mathcal{P} \rightarrow \mathcal{Q}$) and the inverse transformation ($\mathcal{Q} \rightarrow \mathcal{P}$). This double-check regime is not redundant. As discussed in the Sec.\ref{sec:svc}, the SVC can utilize both overlap and non-overlap information. The double-check regime can fully utilize the non-overlapping part of both $\mathcal{P}$ and $\mathcal{Q}$. Sometimes the transformation is good but the inverse transformation could be bad, and we can still reject this transformation. Further analysis of time efficiency and registration performance will be presented in Sec.\ref{sec:analysis_experiments}.
\begin{algorithm}[htb]
  \caption{SVC-based evaluation Algorithm}
  \label{alg:evaluation}
  \KwIn{Input Point clouds $\mathcal{P}$,$\mathcal{Q}$; \\\ \ \ \ \ \ \ \ \ \  
   Correspondences $\mathcal{C}: \{c_i=\{p_i,q_i\}\}$;\\ \ \ \ \ \ \ \ \ \ \ Transformation Hypotheses $\mathcal{T}:\{T_0, T_1,..., T_K\}$}
  \KwOut{Optimal transformation $T^*$}
  for every $T_i$ in $\mathcal{T}$:\\
  \ \ \ \ compute the score of $d_{IC}(\mathcal{C})$ according to Eq.\ref{eq:detailed_IC_metric} \\
  Arrange $\mathcal{T}$ in descending order based on IC score\\
  $T^*$ = $T_0$ \\
  for every $T_i$ in ordered $\mathcal{T}$:\\
  \ \ \ \ if $SVC(\mathcal{P},\mathcal{Q},T_i)$ and $SVC(\mathcal{Q}, \mathcal{P},inverse(T_i))$ is True:\\
  \tcp{Let $T_i=(\bf{R},\bf{t})$, then $inverse(T_i)=(\bf{R^{T}},\bf{-R^Tt})$}
  \ \ \ \ \ \ \ \ $T^* = T_i$ \\
  \ \ \ \ \ \ \ \ break the loop\\
  return  $T^*$\\
\end{algorithm}

\begin{algorithm}[!h]
  \caption{SVC Algorithm}
  \label{alg:SVC}
  \KwIn{Input Point clouds $\mathcal{P}$,$\mathcal{Q}$; \\\ \ \ \ \ \ \ \ \ \  Estimated Transformation $T$}

  \KwOut{True or False}
  \tcp{Step 1. Gain the non-overlap source point clouds.}
  Get the transformed $\mathcal{\hat{P}}$ using $T$\\
  Get the $\mathcal{\hat{P}}_{non}$ according to Eq.\ref{eq:non_overlap_P}\\
  \tcp{Step 2. Project 3D points to a sphere.}
  Get the $\mathcal{\hat{P}}_{sphere}$ and $\mathcal{{Q}}_{sphere}$ in Eq.\ref{eq:projection} according to Eq.\ref{eq:projection_function}\\
  \tcp{Step 3. Calculate the BC metric.}
  Get the $\mathcal{Q}_{roi}$ according to Eq.\ref{eq:Q_roi}, and store every $Proj(\boldsymbol{\hat{p}_j})^*$ in Eq.\ref{eq:NN_of_projection} \\
  Calculate $d_{BC}(\mathcal{\hat{P}}, \mathcal{Q}_{roi})$ according to Eq.\ref{eq:BC_metric}\\

  if BC score < threshold:\\
  \ \ \ \ return True\\
  else:\\
  \ \ \ \ return False\\
\end{algorithm}





\section{Experiments and Results}
\subsection{Datasets and Experimental Setup}
\subsubsection{Datasets} 
To fully validate the effectiveness of the SVC, we conduct the PCR task on both indoor and outdoor scenes. For the indoor scene, we use 3DMatch (1623 pairs) \cite{zeng20173dmatch} and 3DLoMatch (1781 pairs) \cite{huang2021predator} benchmarks. For the outdoor scene, we also follow \cite{chen2022sc2, zhang20233d} and use the provided benchmark based on KITTI \cite{geiger2012we} dataset. To fairly compare our method with previous works, we use the same FPFH \cite{rusu2009fast} and FCGF \cite{choy2019fully} descriptors as \cite{chen2022sc2, zhang20233d}.

\subsubsection{Evaluation Criteria}
Following previous works \cite{chen2022sc2, zhang20233d}, the primary indicator is
registration recall (RR) under an error threshold. For indoor scenes, the threshold is set to (15 deg, 30 cm). For outdoor scenes is (5 deg, 60 cm). The quantify transformation error is also considered. For a pair of point clouds, we compute the isotropic
rotation error (RE) and L2 translation error (TE) as follows:

\begin{equation}
RE(\mathbf{\hat{R}})=acos (\frac{trace(\mathbf{\hat{R}^T}\mathbf{R^*})-1}{2}),\ TE(\mathbf{\hat{t}})=||\mathbf{\hat{t}}-\mathbf{t^*}||_2.
  \label{eq:error_metric}
\end{equation}
Here $\mathbf{R^*}$ and $\mathbf{t^*}$ denote the ground-truth rotation and translation. 

\subsubsection{Implementation Details}
Theoretically speaking, the SVC can combine with all correspondence-based PCR methods as long as they follow the basic pipeline as shown in Fig.\ref{fig:overall_pipeline}. Considering the registration and real-time performance, we use the SC2 \cite{chen2022sc2} to generate estimated transformations then use SVC-based evaluation to select the best result. For our algorithm, we set the hypotheses number $K=200$, $T_{threshold}=0.99997$, $\eta_2 = 0.02$, $\tau=0.1$ for the indoor scene, and $\tau=0.6$ for outdoor scene. 
All experiments were conducted on an Intel i7-12650H CPU and NVIDIA RTX4060 laptop GPU.

\begin{table*}[!b]
  \raggedright
  \caption{Quantitative Results on 3DMatch \& 3DLoMatch dataset.}
  \label{tab:results_on_3DMatch}
  \begin{tabular}{lccccccccccccc}
    \toprule
    & & \multicolumn{2}{c}{3DMatch FPFH} & \multicolumn{2}{c}{3DMatch FCGF} & \multicolumn{2}{c}{3DLoMatch FPFH} & \multicolumn{2}{c}{3DLoMatch FCGF} & \\
    \cmidrule(lr){3-4} \cmidrule(lr){5-6} \cmidrule(lr){7-8} \cmidrule(lr){9-10}
     & &  & RE(deg) &  & RE(deg) &  & RE(deg) &  & RE(deg) &Time\\
     & & RR(\%) & /TE(cm) & RR(\%) & /TE(cm) & RR(\%) & /TE(cm) & RR(\%) & /TE(cm) & (s)\\
    \midrule
    \multicolumn{10}{l}{\textbf{Deep Learned}} \\
    3DRegNet \cite{pais20203dregnet} & & 26.31 & 3.75/9.60 & 77.76 & 2.74/8.13 & - & -/- & - & -/- & 0.07\\
    DGR \cite{choy2020deep}& & 32.84 & 2.45/7.53 & 88.85 & 2.28/7.02 & 19.88 & 5.07/13.53 & 43.80 & 4.17/10.82 & 1.53\\
    DHVR \cite{lee2021deep}& & 67.10 & 2.78/7.84 & 91.93 & 2.25/7.08 & - & -/- & 54.41 & 4.14/12.56 &3.92\\
    PointDSC \cite{bai2021pointdsc} & & 77.39 & \textbf{2.05/6.43} & 92.85 & 2.05/\textbf{6.50} & 27.74 & 4.11/10.45 & 55.36 & 3.79/{10.37} &0.10 \\
    \midrule
    \multicolumn{10}{l}{\textbf{Traditional}} \\
    SM \cite{leordeanu2005spectral} & & 55.88 & 2.94/8.15 & 86.57 & 2.29/7.07 & 6.06 & 6.19/12.62 & 33.52 & 4.28/11.01 & \textbf{0.03}\\
    RANSAC \cite{fischler1981random} & & 65.29 & 3.52/10.98 & 89.62 & 2.50/7.55 & 15.34 & 6.05/13.74 & 46.38 & 5.00/13.11 & 0.97\\
    GC-RANSAC \cite{barath2018graph}& & 71.97 & 2.43/7.03 & 89.53 & 2.25/6.93 & 17.46 & 4.43/10.75 & 41.83 & 3.90/10.44 & 0.55\\
    TEASER \cite{yang2020teaser}& & 75.79 & 2.43/7.24 & 87.62 & 2.38/7.44 & 25.88 & 4.83/11.71 & 42.22 & 4.65/12.07 & 0.07\\
    FGR \cite{zhou2016fast}& & 40.91 & 4.96/10.25 & 78.93 & 2.90/8.41 & - & -/- & 19.99 & 5.28/12.98 & 0.89\\
    SC2 \cite{chen2022sc2} & & 83.26 & {2.09}/6.66 & 93.16 & 2.09/6.51 & 38.46 & 4.04/{10.32} & 58.62 & 3.79/{10.37} & 0.11\\ 	 			 
    SC2++ \cite{chen2023sc}& & 87.18 & 2.10/{6.64} & 94.15 & 2.04/\textbf{6.50} & 41.27 & \textbf{3.86/10.06}  & 61.15 & {3.72}/10.56 & 0.28\\
    MAC \cite{zhang20233d}& & 83.92 & 2.11/6.79 & 93.72 & \textbf{2.03}/6.53 & 41.27 & 4.06/10.64 & 60.08 & {3.75}/10.60 & 1.87\\
    VDIR \cite{zhang20233d}& & - & -/- & - & -/-& - & -/- &64.66 & \textbf{2.59}/\textbf{9.27} & 0.97\\
    \midrule
    Ours & & \textbf{88.66} & 2.18/6.87 & \textbf{94.58} & 2.07/6.60 & \textbf{45.76} & 4.04/10.62 & \textbf{67.77} & 3.93/10.83 & 0.25\\
    \bottomrule
  \end{tabular}
\end{table*}

\subsection{Results on Indoor Scenes} \label{subsec:indoor_scenes}
We perform extensive comparisons based on the 3DMatch \& 3DLoMatch benchmark. Both deep-learned and geometric-only PCR methods are considered, e.g. 3DRegNet \cite{pais20203dregnet}, DGR \cite{choy2020deep}, DHVR \cite{lee2021deep}, PointDSC \cite{bai2021pointdsc}, SM \cite{leordeanu2005spectral}, RANSAC \cite{fischler1981random}, GC-RANSAC \cite{barath2018graph}, 
TEASER \cite{yang2020teaser}, FGR \cite{zhou2016fast}, SC2 \cite{chen2022sc2}, MAC \cite{zhang20233d}, SC2++ \cite{chen2023sc} and VDIR \cite{xing2024efficient}. To fully test the outlier rejection performance of PCR methods. Since the VDIR \cite{xing2024efficient} only maintains 1279 pairs of the 3DMatch benchmark (1623 pairs). To make a fair comparison, we only compare it under the 3DLoMatch benchmark. The quantitative results are shown in Tab.\ref{tab:results_on_3DMatch}.

\subsubsection{Results on 3DMatch}
As shown in Tab.\ref{tab:results_on_3DMatch}, for the most important criterion registration recall (RR), SVC-based evaluation boosts the performance of SC2 from (83.92\% \& 94.15\%) to (88.66\% \& 94.58\%). This result also slightly outperforms the SC2++. For the transformation error criterion, there is no improvement even worse on the rotation error (RE) and translation error (TE). This is because the SVC only determines whether an estimated transformation is correct, it does not produce new hypotheses. So the transformation error mainly depends on the SC2.

\subsubsection{Results on 3DLoMatch}
The 3DLoMatch is a low overlap rate (10\% to 30\%) version of 3DMatch, which is still challenging in the PCR field. As shown in Tab.\ref{tab:results_on_3DMatch}, it is obvious that the registration performance of all PCR methods is much worse compared with the 3DMatch benchmark. Under this condition, our method greatly improves the registration recall (RR) from (38.46\% \& 58.62\%) to (45.76\% \& 67.77\%), which also greatly outperforms all other methods by at least 4\% \& 3\% improvement.

It is noteworthy that our method exhibits a more substantial improvement on 3DLoMatch in comparison to 3DMatch. Two main facts could contribute to that: (1) Due to the low baseline registration recall, there is ample room for improvement. (2) As discussed in the Methods section, the SVC mainly utilizes the non-overlap part to distinguish the false transformation. Under the low overlap condition, the SVC can better narrow the range of transformation to achieve more robust results.

\begin{table}[!b]
  \centering
  
  \caption{Comparison of Methods on Geotransformer \cite{10076895}}
  \small
  \label{tab:geotransformer_comparison}
  \begin{tabular}{lcccc}
    \toprule
    & \multicolumn{2}{c}{3DMatch} & \multicolumn{2}{c}{3DLoMatch} \\
    \cmidrule(lr){2-3} \cmidrule(lr){4-5}
     & &  RE(deg) & &  RE(deg)\\
    Method & RR(\%) & TE(cm) & RR(\%)  & TE(cm) \\
    \midrule
    LGR & 92.70 & 1.81/6.30 & 75.00 & \textbf{2.94}/9.10 \\
    MAC & 95.70 & -/- & 78.90 & -/- \\
    SC2 & 96.06 & \bf{1.63}/\bf{5.57} & 78.11 & 3.01/8.69 \\
    SC2++ & - & -/- & 78.72 & 2.96/\bf{8.56} \\
    VDIR & - & -/- & 79.50 & -/- \\
    \midrule
    Ours & \bf{97.29} & 1.67/{5.65} & \bf{82.37} & 3.09/8.90 \\
    \bottomrule
  \end{tabular}
\end{table}

\subsubsection{Combined with Geotransformer} Since the SVC is for general PCR transformation selection, it can be easily combined with deep learning based frameworks and take effect. Currently, GeoTransformer \cite{10076895} represents SOTA performance for correspondence learning on 3DLoMatch dataset. We compare the registration results of our method with recent outlier removal methods including MAC \cite{zhang20233d}, SC2 \cite{chen2022sc2}, SC2++ \cite{chen2023sc} and VDIR \cite{xing2024efficient}. The original GeoTransformer uses the LGR \cite{10076895} as the outlier removal method, so we consider it as the baseline. As discussed in Sec.\ref{subsec:indoor_scenes}, we only compare the VDIR on the 3DLoMatch benchmark. As shown in Tab.\ref{tab:geotransformer_comparison}, our method boosts the registration recall (RR) of the SC2 from (96.06\% \& 78.11\%) to (97.29\% \& 82.37\%). Especially for the 3DLoMatch dataset, our method improves the RR criterion by about 3\% over the closest competitors.

\subsection{Results on Outdoor Scenes}
To check the generalization performance of the SVC on outdoor scenes, we also perform experiments on the KITTI dataset. As shown in Tab.\ref{tab:kitti_comparison}, our method achieves competitive registration results on the KITTI dataset. The SC2 already performs very well (around 99\% RR) on the KITTI dataset, so there is little improvement when combined with SVC. For the FCGF condition, our method improves the registration recall from 98.20\% to 98.74\% and slightly reduces the translation error (TE). These improvements indicate that our method produces an effect and is suitable for outdoor scenes. This experiment also indicates that the SVC will not worsen the registration performance which is consistent with the theory in Methods.

\begin{table}[htbp]
  \centering
  \caption{Quantitative Results on KITTI dataset.}
  \small
  \label{tab:kitti_comparison}
  \begin{tabular}{lcccc}
    \toprule
    & \multicolumn{2}{c}{FPFH} & \multicolumn{2}{c}{FCGF} \\
    \cmidrule(lr){2-3} \cmidrule(lr){4-5}
    & &  RE(deg) & &  RE(deg)\\
    Method & RR(\%) & TE(cm) & RR(\%)  & TE(cm) \\
    \midrule
    \multicolumn{5}{l}{\textbf{Deep Learned}} \\
    DGR & 77.12 & 1.64/33.10 & 98.20 & 0.34/21.70 \\
    DHVR & - & -/- & \bf{99.10} & \bf{0.29}/{19.80} \\
    PointDSC & 98.20 & 0.35/8.13 & 98.02 & 0.33/21.03 \\
    \midrule
    \multicolumn{5}{l}{\textbf{Traditional}} \\
    FGR & 5.23 & 0.86/43.84 & 89.54 & 0.46/25.72 \\
    RANSAC & 74.41 & 1.55/30.20 & 98.02 & 0.39/23.17 \\
    SC2 & \bf{99.64} & {0.34}/7.81 & 98.20 & 0.33/20.76 \\
    SC2++ & \bf{99.64} & \bf{0.32}/\bf{7.19} & 98.56 & {0.32}/20.61 \\
    MAC & 99.46 & 0.40/8.46 & 97.84 & 0.34/\bf{19.34} \\ 	 	 	 	 	 
    \midrule
    Ours & \bf{99.64} & {0.34}/{7.77} & {98.74} & 0.33/20.53 \\
    \bottomrule
  \end{tabular}
\end{table}

\subsection{Analysis Experiments}
\label{sec:analysis_experiments}
In this section, we perform a time efficiency and registration performance analysis of the SVC. 

\begin{figure*}[tb]
  \centering
  \subfigure[\label{fig:time-a} Running time vs \# of points.]{\includegraphics[width=0.49\linewidth]{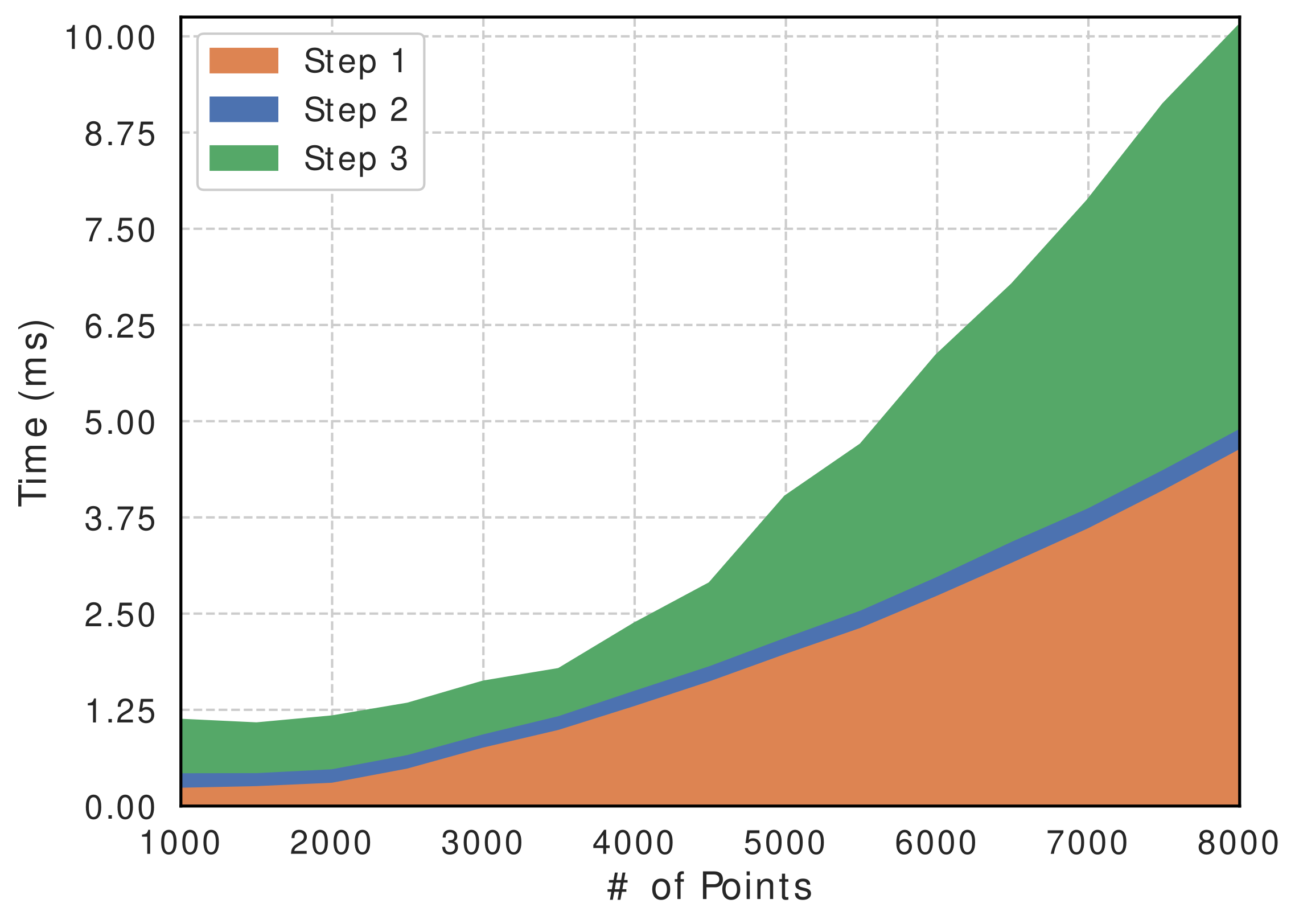}}
  \hfill
  \subfigure[\label{fig:time-b} Average running time vs \# of hypotheses.]{\includegraphics[width=0.48\linewidth]{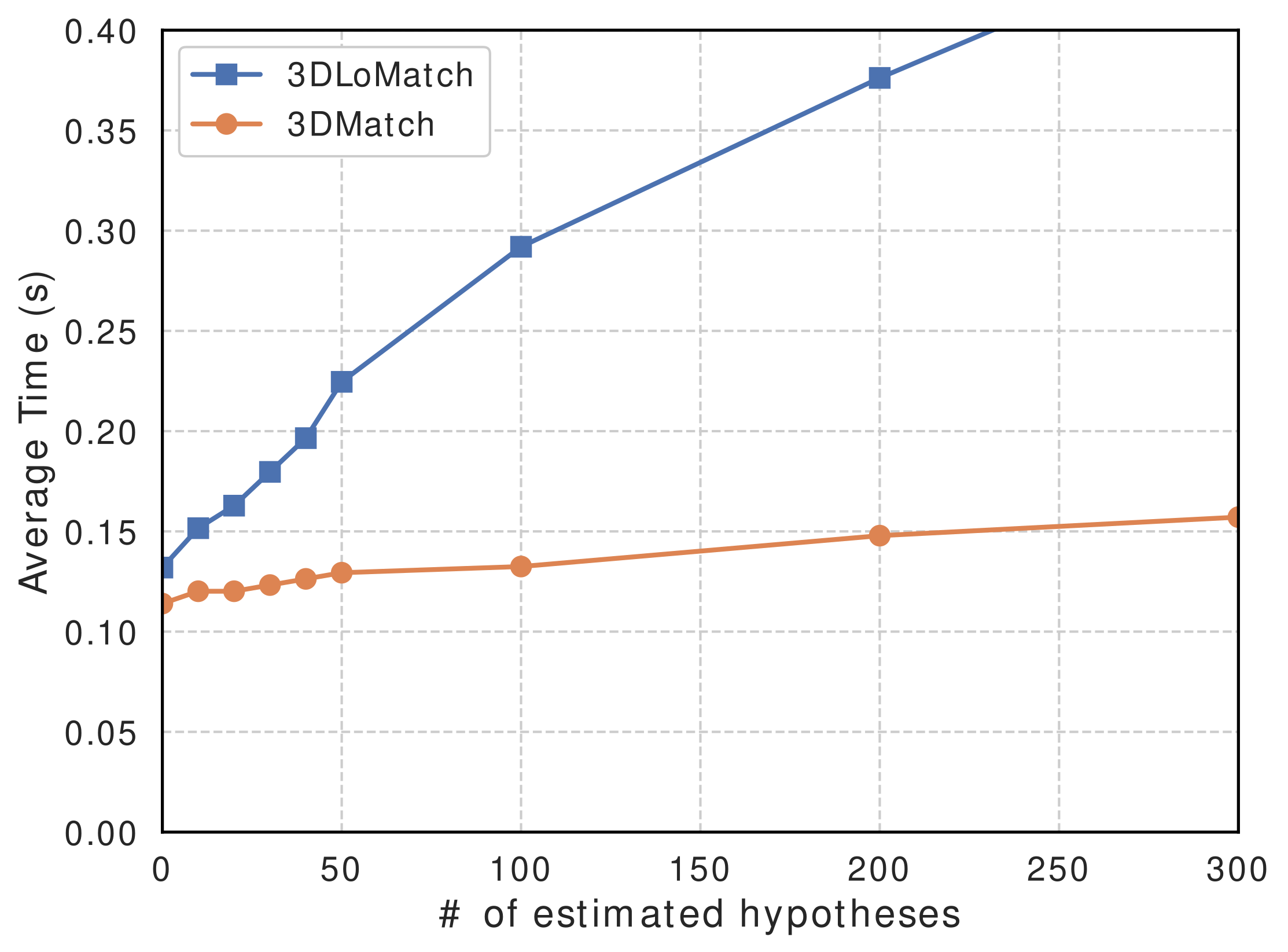}}
    
    
  \caption{\textbf{Time efficiency analysis.} Fig.\ref{fig:time-a} is about one single execution of the SVC algorithm. Fig.\ref{fig:time-b} is about the average time of the SC2 combined with SVC on different datasets.}
  \label{fig:time}
\end{figure*}

\subsubsection{Time efficiency analysis}
As shown in the Fig.\ref{fig:time-a},
We first evaluate the time changing when the number of points rises. The results are shown in Fig.\ref{fig:time-a}, and the following conclusion can be made:
(1) In general, the SVC is very time-efficient. Even with an input of 8000 points, the time consumption of the SVC is about 10ms. (2) It is obvious that Steps 1 and 3 
 occupy a significant amount of time to find the nearest neighbor for each point in point clouds. They both involve multiplying two big matrices like $\bf{A}\bf{B}^\text{T}$, both $\bf{A}$ and $\bf{B}$ are large, dense $N \times 3$ matrix. This large matrix multiplication is the primary bottleneck of our algorithm and the time complexity is $\mathcal{O}(N^2)$. For the CPU-only version, we can use kdtree in PCL library \cite{rusu20113d} to reduce the time complexity to $\mathcal{O}(Nlog(N))$, the results are shown in Tab.\ref{tab:execution_time_on_c++}. It is clear that our algorithm is also time efficient only using the CPU and the growth in time is relatively slow as the number of points increases.
\begin{table}[!b]
    \raggedright
    \caption{Execution Time on CPU only}
    \small
    \begin{tabular}{cccccc}
        \toprule
        \# of points & 2000 & 4000 & 6000 & 8000 & 10000   \\
        \midrule
        Time (ms) & 5.21 & 6.97 & 10.44 & 13.09 & 16.28 \\
        \bottomrule
    \end{tabular}
    \label{tab:execution_time_on_c++}
\end{table}
Since the SVC could run multiple times in actual program execution, so we also evaluate the SC2+SVC actual time consumption on 3DLoMatch \& 3DMatch datasets. The SC2 \cite{chen2022sc2} could generate hundreds even thousands of transformation hypotheses, we assess the runtime variation based on the number of hypotheses that could be processed by the SVC. It represents the original SC2 when the number is 0. According to the results in Fig.\ref{fig:time-b}, the following conclusion can be made:

(1) The time-changing varies between 3DMatch and 3DLoMatch. For 3DMatch, the time increases very slowly. This is mainly due to the SC2 already performing well in this dataset so that the SVC could recognize the reasonable hypothesis in a few iterations for most instances. For the challenging 3DLoMatch datasets, our method needs to check every hypothesis to find the optimal transformation for a large number of instances. (2) In general, the extra time cost of our method is not significant. Even on the 3DLoMatch dataset and running a hundred iterations, the time cost remains within the same order of magnitude. Please note not all hypotheses must be evaluated, for example, the SC2 may generate a large amount of close transformations. If we can make a reasonable selection of these hypotheses, the iterations could reduce significantly without registration performance compromising.

\subsubsection{Registration performance analysis} The number of hypotheses directly affects the registration performance. We perform the registration recall analysis experiment on 3DMatch \& 3DLoMatch datasets combined with FCGF \cite{choy2019fully}. Results are shown in Fig.\ref{fig:RR}, the following conclusions can be made: (1) In general, the registration recall increases along with the \# of hypotheses until convergence. The rate of increase is fast when the \# of hypotheses is under 100. (2) It also implies that our method will not decrease the registration performance when combined with SC2. This is mainly due to our algorithm being strict and conservative. (3) There is a significant difference in growth ranges between the two datasets. For the 3DLoMatch the range is from 58\% to 68\%, while for the 3DMatch is 93\% to 95\%.

\begin{figure*}[hb]
  \centering
  \subfigure[\label{fig:RR-a} 3DLoMatch with FCGF]{\includegraphics[width=0.48\linewidth]{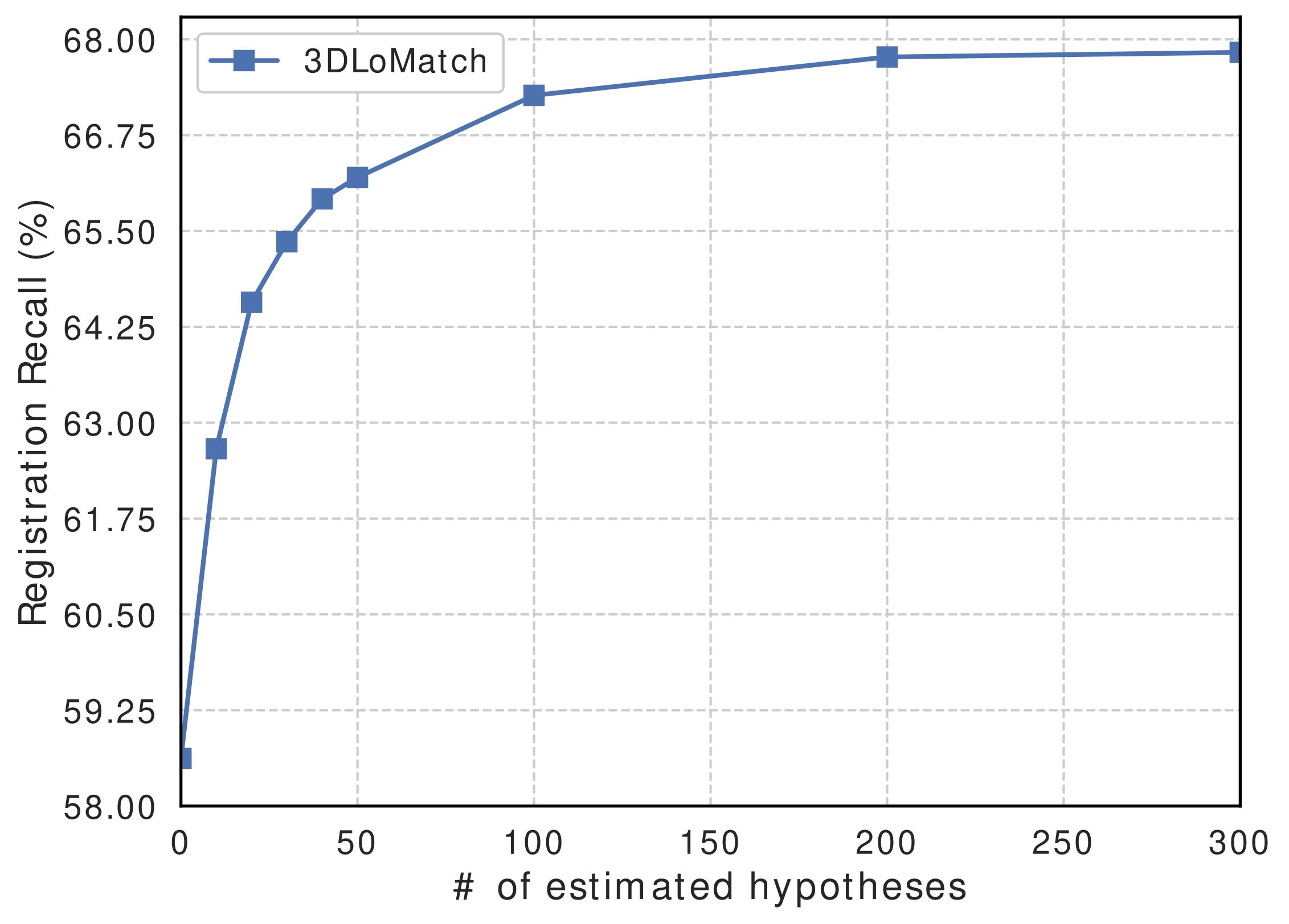}}
  \hfill
   \subfigure[\label{fig:RR-b} 3DMatch with FCGF]{\includegraphics[width=0.48\linewidth]{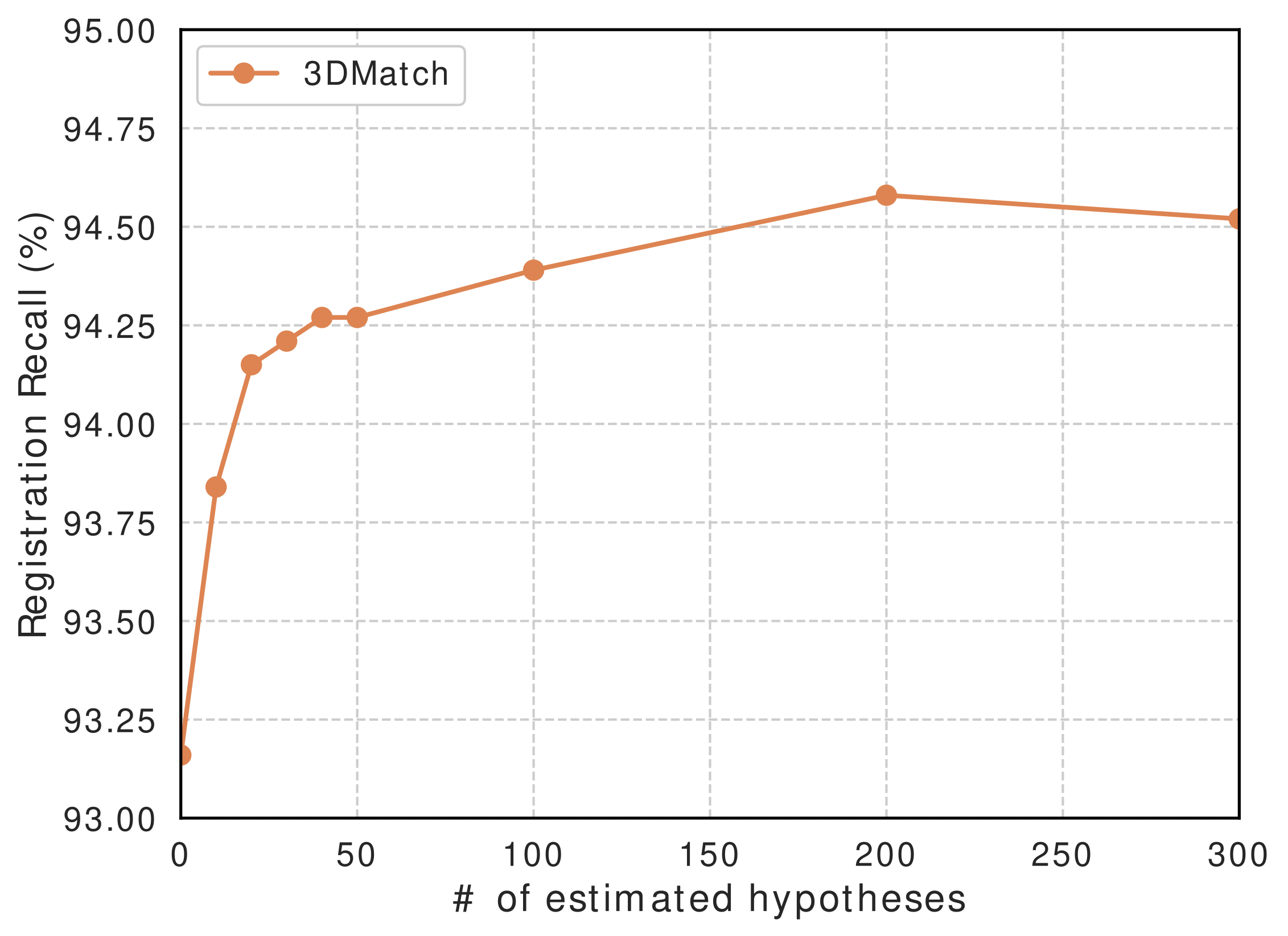}}
    
  
    
  \caption{Registration recall changes with different numbers of estimated hypotheses.}
  \label{fig:RR}
\end{figure*}

\subsection{The decision version of partial PCR problem}
\label{subsec:decision_task}
The decision version of the PCR problem is to determine whether the given transformation for input point clouds is correct. As we discussed in the Method section, the objective of partial to partial PCR is still undefined. This means the decision version of the partial PCR problem is still an open problem. As far as the best we know, there is a lack of research on this task and no related benchmarks.

So we conduct a simple experiment to validate the performance of the SVC on this task. First, we use the SC2 combined with FCGF and FPFH to generate estimated transformations on 3DMatch \& 3DLoMatch datasets, thus this benchmark includes 1623 pairs + 1781 pairs with each pair having two estimated transformations. We use SVC to classify positive and negative transformations and utilize the F1 score to assess the performance of classification models. We set SC2 as the baseline. Its recall is 100\% since it considers all generated transformations as optimal. According to the Tab.\ref{tab:recision_and_recall}, it is obvious that SVC outperforms the baseline by a big margin. Please note this experiment can only qualitatively compare SVC with the baseline since the benchmark we use is insufficient.

Let's further contemplate the relationship between this decision version task and the original PCR task. For any problem that can be reliably solved, we can certainly validate its results reliably. So the decision version task is the basis of the PCR task and needs to be resolved ahead of the PCR task. Assuming this decision version task can be resolved readily, then the following problem is how to generate hypotheses that contain correct transformation. In conclusion, we redefine the PCR regime as (1) effectively generating hypotheses with the correct one and (2) identifying the correct transformation readily.

\begin{table}[!t]
  \raggedright
  \caption{Classification performance on a simple benchmark.}
  \label{tab:recision_and_recall}
  \begin{tabular}{lcccc}
    \toprule
    & \multicolumn{2}{c}{3DLoMatch} & \multicolumn{2}{c}{3DMatch} \\
    \cmidrule(lr){2-3} \cmidrule(lr){4-5}
    &  P/R & F1 score & P/R & F1 score\\
    \midrule
    Baseline & 48.5/100.0 & 65.4 & 88.5/100.0 & 93.9 \\
    SVC & 88.7/88.3 & \bf{88.5} & 98.1/95.7 & \bf{96.9} \\
    \bottomrule
  \end{tabular}
\end{table}

\section{Conclusion}
\label{sec:conclusion}
In this paper, we introduce a novel and general Sight View Constraint (SVC) for robust point cloud registration tasks. The method significantly improves the ability to identify the correct transformation of existing PCR metrics, especially for the low-overlap condition. Extensive experiments show that our method outperforms on multiple datasets. By further analysis of the partial PCR task, we highlight the importance of the decision version of the partial PCR task, which has the potential to provide novel insights into the research problem.

\bibliographystyle{unsrt}
\bibliography{refs,refs1}
\end{document}